\documentclass[12pt]{arxiv} %

\newcommand{\Real}{\mathbb R}
\newcommand{\eps}{\varepsilon}
\newcommand{\abs}[1]{\left\vert#1\right\vert}
\newcommand{\norm}[1]{\left\Vert#1\right\Vert}
\newcommand{\ra}{\rightarrow}
\usepackage{mathtools}
\usepackage{multirow}
\DeclarePairedDelimiter\inner{\langle}{\rangle}
\DeclareMathOperator{\E}{\mathbb{E}}
\renewcommand{\set}[1]{\left\{#1\right\}}

\newcommand{\U}{\mathcal U}
\renewcommand{\S}{\mathcal S}
\newcommand{\M}{\mathcal M}
\renewcommand{\O}{\mathcal O}

\title[Almost Sure Saddle Avoidance of Stochastic Gradient Methods]{\bfseries Almost Sure Saddle Avoidance of Stochastic Gradient Methods without the Bounded Gradient Assumption}
\usepackage{times}

\author{%
 \Name{Jun Liu} \Email{j.liu@uwaterloo.ca}\\
 \addr Department of Applied Mathematics, University of Waterloo, Waterloo, Canada
 \AND
 \Name{Ye Yuan} \Email{yye@hust.edu.cn}\\
 \addr School of Artificial Intelligence and Automation \& School of Mechanical Science and Engineering, Huazhong University of Science and Technology, Wuhan, China
}

\begin{document}

\maketitle

\begin{abstract}%
We prove that various stochastic gradient descent methods, including the stochastic gradient descent (SGD), stochastic heavy-ball (SHB), and stochastic Nesterov's accelerated gradient (SNAG) methods, almost surely avoid any strict saddle manifold. To the best of our knowledge, this is the first time such results are obtained for SHB and SNAG methods. Moreover, our analysis expands upon previous studies on SGD by removing the need for bounded gradients of the objective function and uniformly bounded noise. Instead, we introduce a more practical local boundedness assumption for the noisy gradient, which is naturally satisfied in empirical risk minimization problems typically seen in training of neural networks.
\end{abstract}


\section{Introduction}

Stochastic gradient descent (SGD) and its variants, such as the stochastic heavy-ball (SHB) (i.e., stochastic gradient descent with momentum) and stochastic Nesterov's accelerated gradient (SNAG) methods, have emerged as popular optimization algorithms in machine learning and other computational fields. Despite their widespread use, a fundamental question that remains is whether these methods are effective at avoiding strict saddle manifolds, which can hinder convergence to optimal solutions. In this paper, we address this question and demonstrate that SGD, SHB, and SNAG almost surely avoid strict saddle manifolds.

To the best of our knowledge, this is the first time such results have been obtained for SHB and SNAG methods. Furthermore, our analysis expands upon previous studies by removing the requirement for bounded gradients in the objective function and replacing it with a more practical local boundedness assumption for the noisy gradient, which is commonly satisfied in empirical risk minimization problems such as in neural network training. Our findings provide valuable insights into the behavior of these algorithms and have implications for their use in various computational tasks.

\subsection{Related work} 

For deterministic gradient descent methods, \cite{lee2016gradient,lee2019first} proved that with a step size smaller that $1/L$, where $L$ is the Lipschitz constant of the gradient, gradient descent always avoids strict saddles unless initialized on a set of measure zero (i.e., the stable manifold of the saddles). Various extensions of this result were made, with different assumptions on the gradient oracle, choice of step sizes, and structure of the saddle manifold. Readers are referred to \cite{du2017gradient,vlatakis2019efficiently,jin2017escape,lee2016gradient,lee2019first,panageas2017gradient,panageas2019first} and references therein. 

For saddle point avoidance by stochastic gradient methods, early work by \cite{pemantle1990nonconvergence} and \cite{brandiere1996algorithmes} in the context of stochastic approximations showed that standard SGD almost surely avoids hyperbolic saddle points, i.e., points $x^*$ such that $\lambda_{\min}(\nabla^2 f(x^*))<0$ and $\text{det}(\nabla^2 f(x^*))\neq 0$. The work by \cite{benaim1995dynamics} proved almost sure avoidance of hyperbolic linearly unstable cycle by SGD. Later work by \cite{brandiere1998some,benaim1999dynamics} extended such results to show that SGD-type algorithms almost surely avoid more general repelling sets. More recently, using different techniques and under different assumptions, \cite{ge2015escaping} showed that SGD avoids strict saddles points satisfying $\lambda_{\min}(\nabla^2 f(x^*))<0$ with high probability. More specifically, they showed that with a constant step size $\eta$, SGD produces iterates close to a local minimizer and hence avoids saddle points, with probability at least $1-\zeta$, after $\Theta(\log(1/\zeta)/\eta^2)$ iterations. The work of \cite{daneshmand2018escaping,fang2019sharp} further obtained results on high-probability avoidance of saddle points and convergence to second-order stationary points, while the more recent work by \cite{vlaski2022second} proved efficient escape from saddle points under expectation. 

The work closest to ours is that of \cite{mertikopoulos2020almost}, in which the authors proved that SGD almost surely avoids any strict saddle manifold for a wide spectrum of vanishing step size choices, following earlier work by \cite{pemantle1990nonconvergence,benaim1995dynamics,benaim1999dynamics}. However, in these works, it is always assumed that the noise on the stochastic gradient is bounded. Moreover, while making an effort to circumvent the bounded trajectory assumption in prior work, \cite{mertikopoulos2020almost} also assumed that the objective function is $G$-Lipschitz, which means the gradient is always bounded. This is, however, a very strong assumption, as even quadratic objective functions do not satisfy it. 

To the best of the authors' knowledge, our paper is the first to show that the SHB and SNAG methods almost surely avoid saddle points. Our work also sharpens the analysis for SGD by removing the bounded gradient assumption and relaxing the bounded noise assumption to a local boundedness assumption, which is always satisfied in empirical risk minimization problems. The key ingredient required to achieve our results was provided by \cite{liu2022almost}, who showed that both SHB and SNAG almost surely produce iterates with gradients converging to zero, even in the non-convex setting under very weak assumptions \citep{khaled2020better} on the stochastic gradient. This almost sure convergence of the gradient, combined with the same asymptotic non-flatness assumption on the objective function as in \cite{mertikopoulos2020almost}, allowed us to circumvent the bounded gradient and bounded noise assumptions.

\section{Problem formulation and assumptions}

\subsection{Optimization problem setup}

Consider the unconstrained minimization problem 
\begin{equation}
	\label{eq:min}
	\min_{x\in\Real^d} f(x), 
\end{equation}
where $f:\,\Real^d\ra\Real$ is a potentially non-convex function. We aim to solve it through stochastic gradient methods. For example, the function $f$ can represent either an expected risk of the form $f(x)=\mathbb{E}[f(x;\xi)]$, where $\xi$ is a random sample or set of samples, or an empirical risk of the form $f(x)=\frac{1}{n}\sum_{i=1}^nf_i(x;\xi_i)$, where $
\set{\xi_i}_{i=1}^n$ are realizations of $\xi$. The following assumptions are made.

\begin{assumption}[$L$-smoothness]\label{as:smoothness}
	The function $f:\,\Real^d\ra\Real$ is %
	is smooth  
	and its gradient $\nabla f$ is $L$-Lipschitz, i.e., there exists a constant $L>0$ such that $$\norm{\nabla f(x)-\nabla f(y)}\le L\norm{x-y}$$ for all $x,y\in\Real^d$. 
\end{assumption}

By ``smoothness'', we require that $\nabla f$ is $C^2$ so that we can apply the center manifold theorem to the gradient flow dynamics $\dot{x}=-\nabla f(x)$. This is central to the argument for avoiding saddles \citep{pemantle1990nonconvergence,benaim1999dynamics,mertikopoulos2020almost}. The usual convergence analysis of gradient descent often only requires that $f$ is continuously differentiable. 
It is well known that Assumption \ref{as:smoothness} implies the following (see, e.g., \citet[Lemma 1.2.3]{nesterov2003introductory}) 
\begin{equation}
	\label{eq:L}
	f(y) \le f(x) + \inner{\nabla f(x),y-x} + \frac{L}{2}\norm{y-x}^2,\quad \forall x,y\in\Real^d.
\end{equation}

\begin{assumption}[Asymptotic non-flatness]\label{as:nonflat}
	The function $f:\,\Real^d\ra\Real$ is not asymptotically flat in the sense that $\liminf_{\norm{x}\ra\infty}\norm{\nabla f(x)}>0$. 
\end{assumption}

These assumptions are fairly standard for non-convex optimization. Assumption \ref{as:smoothness} essentially implies that gradient descent with sufficiently small step sizes is well behaved. 
Assumption \ref{as:nonflat} ensures that once gradient descent converge to a critical set on which gradient vanishes, then this set can be bounded away from infinity.

\subsection{Stochastic gradient oracle}

We assume that we have access to an unbiased stochastic gradient oracle which produces $\nabla f(x;\xi)$ satisfying $\E[\nabla f(x;\xi)]=\nabla f(x)$ for each $x\in \Real^n$. We further assume that the stochastic gradient satisfies the following expected smoothness condition \citep{khaled2020better}. 

	\begin{assumption}[ABC condition]\label{as:abc}
		There exist nonnegative constants $A$, $B$, and $C$ such that 
		\begin{equation}\label{eq:abc}
			\E[\norm{\nabla f(x;\xi)}^2] \le A(f(x)-f^*) + B\norm{\nabla f(x)}^2 + C,\quad \forall x\in \Real^d. 
		\end{equation}
	\end{assumption}

The above assumption was proposed in \cite{khaled2020better} as ``the weakest assumption'' for analysis of stochastic gradient descent in the non-convex setting. We further make the following local boundedness assumption on the stochastic gradient.

	\begin{assumption}[Local boundedness]\label{as:localb}
		For each compact set $K\in \Real^n$, there exists a constant $C$
 such that $\norm{\nabla f(x;\xi)}\le C$  for all $x\in K$ almost surely. 
	\end{assumption}

\begin{remark}
The local boundedness assumption is clearly weaker than the assumption of almost surely bounded noise \citep{mertikopoulos2020almost}, i.e., there exists a constant $C$ such that 
$$\norm{\nabla f(x;\xi)-\nabla f(x)}\le C$$ for all $x\in \Real^n$ almost surely. Indeed, since $\nabla f(x)$ is assumed to be (Lipschitz) continuous and hence locally bounded, bounded noise and local boundedness of $\nabla f(x)$ implies local boundedness of $\nabla f(x;\xi)$. Assumption \ref{as:localb} is readily satisfied for 
stochastic gradient computed using a sample drawn from a finite number of samples where each sample gives a gradient function $\nabla f(x;\xi)$ that is locally bounded. For instance, 
let $f(x)=\frac{1}{n}\sum_{i=1}^nf_i(x)$, and suppose that each $\nabla f(x;\xi)$ corresponds to uniformly randomly choosing $i\in \set{1,\ldots,n}$ and computing $\nabla f(x;\xi)=\nabla f_i(x)$. If each $\nabla  f_i(x)$ is locally bounded, then Assumption \ref{as:localb} holds. 
\end{remark}

Finally, we need one more assumption on the stochastic gradient. 

\begin{assumption}\label{as:uniform}
	The error between the true gradient and any stochastic gradient is uniformly exciting in the sense that there exists some constant $b>0$ such that 
	$$
	\E[\inner{\nabla f(x;\xi)-\nabla f(x),v}^+] \ge b,
	$$
	for all $x\in\Real^d$ and all unit vector $v\in\Real^n$. 
\end{assumption}

The same assumption was made in prior work \citep{pemantle1990nonconvergence,benaim1996asymptotic,benaim1999dynamics,mertikopoulos2020almost}. As remarked in \cite{mertikopoulos2020almost}, this assumption is naturally satisfied by noisy gradient dynamics (e.g., as in \cite{ge2015escaping}), as well as finite sum objectives with at least $d$ objectives. 

\subsection{Stochastic gradient descent with momentum}

	The iteration of the stochastic heavy-ball (SHB) method is given by
	\begin{equation}\label{eq:shb1}
		x_{n+1} = x_n - \alpha_n g_n + \beta (x_n - x_{n-1}),\quad n\ge 1,
	\end{equation}
	where $g_n:=\nabla f(x_n;\xi_n)$ is the stochastic gradient at $x_n$, $\alpha_n$ is the step size, and $\beta\in [0,1)$. By convention, we let $x_1$ and $x_{0}$ be the initial conditions. Clearly, if $\beta=0$, SHB reduces to the standard SGD. 
	
	Define
	\begin{equation}\label{eq:zv}
		z_n = x_n - \frac{\beta}{1-\beta} v_n,\quad v_n = x_n -x_{n-1}.
	\end{equation}
	The iteration of SHB can be rewritten as \begin{equation}\label{eq:shb2}
		\begin{aligned}
			v_{n+1} & = \beta v_n -\alpha_n g_n,\\
			z_{n+1} & = z_n - \frac{\alpha_n}{1-\beta} g_n.
		\end{aligned}
	\end{equation}

    Similarly, the iteration of the stochast Nesterov's accelerated gradient (SNAG) method is given by
	\begin{equation}\label{eq:snag}
		\begin{aligned}
			y_{n+1} &= x_n - \alpha_n g_n,\\ 
			x_{n+1} &= y_{n+1} + \beta (x_n - x_{n-1}),
		\end{aligned}
	\end{equation}
	where $g_n:=\nabla f(x_n;\xi_n)$ is the stochastic gradient at $x_n$, $\alpha_n$ is the step size, and $\beta\in [0,1)$. Clearly, if $\beta=0$, SNAG also reduces to the standard SGD. 
	Define $z_n$ and $v_n$ as in (\ref{eq:zv}). The iteration of SNAG can be rewritten as 
	\begin{equation}\label{eq:snag2}
		\begin{aligned}
			v_{n+1} & = \beta v_n -\beta \alpha_n g_n,\\
			z_{n+1} & = z_n - \frac{\alpha_n}{1-\beta} g_n.
		\end{aligned}
	\end{equation}
	Indeed, (\ref{eq:snag2}) is almost identical to (\ref{eq:shb2}) except for the extra $\beta$ in the first equation for $v_{n+1}$.

\section{Stochastic gradient methods avoid strict saddle manifold}

\subsection{Statement of the main result}

Define the critical set as
\begin{equation}
	\label{eq:critical}
	\mathcal{C}=\set{x\in \Real^d:\,\nabla f(x)=0}.
\end{equation}

\begin{definition}[\cite{mertikopoulos2020almost}]\label{def:saddle}
A strict saddle manifold  $\mathcal{S}$ of $f$ is a smooth connected component of $\mathcal{C}$ satisfying 
\begin{enumerate}
	\item Every $x^*\in S$ is a strict saddle point \citep{lee2016gradient,lee2019first} of $f$, i.e., $\lambda_{\min}(\nabla^2 f(x^*))<0$.

	\item For all $x^*\in S$, all negative eigenvalues of $\nabla^2 f(x^*)$ are uniformly bounded from above by a negative constant and all positive eigenvalues of $\nabla^2 f(x^*)$ are uniformly bounded from below by a positive constant. 
\end{enumerate}
\end{definition}

The main result of the paper is stated below. 

\begin{theorem}\label{thm:main}
Let $\mathcal{S}$ be any strict saddle manifold of $f$. Suppose that $f$ satisfies Assumptions \ref{as:smoothness}--\ref{as:uniform}. Then $\mathbb{P}(x_n\ra \mathcal{S}\text{ as }n\ra\infty)=0$ for both SHB (\ref{eq:shb1}) and SNAG (\ref{eq:snag}).
\end{theorem}

\subsection{Outline of the proof}

The rest of this section is dedicated to the proof of Theorem \ref{thm:main}. To help the readers, we provide a brief outline for the proof. 

\begin{enumerate}
	
\item In Section \ref{sec:prior}, we summarize preliminary results on the convergence of the sequences generated by SHB and SNAG \citep{liu2022almost}. Combined with \cite{benaim1996dynamical}, we show that the limit set of these sequences enjoy the same properties of the limit sets of trajectories of the corresponding gradient flow. 

\item In Section \ref{sec:lyap}, we state previous results by \cite{benaim1995dynamics,benaim1999dynamics} on the construction of a Lyapunov function around the saddle manifold. This result will be used later in the proof. 

\item The main proof is presented in Section \ref{sec:analysis}, where the probabilistic estimates by \cite{pemantle1990nonconvergence,pemantle1992vertex} are combined with the Lyapunov analysis due to  \cite{benaim1995dynamics,benaim1999dynamics} to show both SHB and SNAG almost surely avoid strict saddle manifolds, without the bounded gradient and noise assumptions, compared with results on SGD by \cite{mertikopoulos2020almost}.

\end{enumerate}

\subsection{Preliminary convergence results and property of limit sets}\label{sec:prior}

Consider the continuous-time gradient flow for the objective function $f$: 
\begin{equation}
    \label{eq:gdflow}
    \dot{x} = -\nabla f(x).
\end{equation}
Let $\set{\Phi_t}$ denote the flow associated with equation (\ref{eq:gdflow}), i.e., $\Phi_t$ maps any initial condition $x$ to the value of the solution to (\ref{eq:gdflow}) at time $t$, $\Phi_t(x)$.

The following lemma is purely deterministic, but can be used to show limit points of the sequences produced by SHB (\ref{eq:shb1}) and SNAG (\ref{eq:snag}) basically enjoy the same properties as the omega limit sets of the trajectories of the gradient flow (\ref{eq:gdflow}). 

\begin{lemma}[\cite{benaim1996dynamical}] \label{lem:chain}
Let $\set{z_n}$, $\set{u_n}$, and $\set{b_n}$ be sequences in $\Real^d$ such that 
$$
z_{n+1} = z_n + \alpha_n(-\nabla f(z_n) + u_n + b_n),
$$
where $\set{\alpha_n}$ is a positive sequence satisfying $\sum_{n=1}^\infty \alpha_n =\infty$ and $\lim_{n\ra\infty}\alpha_n=0$. Assume:
\begin{enumerate}
    \item $\set{z_n}$ is bounded;
    \item $\lim_{n\ra\infty} b_n=0$;
    and 
    \item for each $T>0$, 
    $$
    \lim_{n\ra \infty} \sup_{\set{k:\,0\le \tau_k-\tau_n\le T}} \norm{\sum_{i=n}^{k-1}\alpha_iu_i} = 0,
    $$
    where $\tau_n=\sum_{i=1}^n \alpha_i$. Then the limit set of $\set{z_n}$ is a nonempty, compact, connected set which is invariant under the flow $\set{\Phi_t}$ of (\ref{eq:gdflow}). Furthermore, the limit set belongs to the chain recurrent set of (\ref{eq:gdflow}). 
\end{enumerate}
\end{lemma}

We recall another lemma that asserts convergence properties of the sequences produced by SHB (\ref{eq:shb2}) and SNAG (\ref{eq:snag2}). 

\begin{lemma}[\cite{liu2022almost}] \label{lem:conv}
    Suppose that Assumptions \ref{as:smoothness} and \ref{as:abc} hold. Furthermore,  $\set{\alpha_n}$ satisfies
$$
		\sum_{n=1}^\infty \alpha_n=\infty, \quad \sum_{n=1}^\infty \alpha_n^2 <\infty. 
		$$
  Then the following results hold:
  \begin{enumerate}
      \item $\nabla f(x_n)\ra 0$, $\nabla f(z_n)\ra 0$, and $v_n\ra 0$, as $n\ra\infty$, almost surely;
      \item $\sum_{i=1}^n \alpha_i(\nabla f(x_n) - g_n)$ is a martingale bounded in $L^2$ and hence converges almost surely.
  \end{enumerate}
\end{lemma}

\begin{proof}
    The proof can be found in the proof of Theorem 4 \citep[Appendix C]{liu2022almost}. 
\end{proof}

Based on these two lemmas, we can  prove the following result. 

\begin{proposition}\label{prop:limset}
Suppose that Assumptions \ref{as:smoothness}, \ref{as:nonflat}, and \ref{as:abc}  and $\set{\alpha_n}$ satisfies
$$
		\sum_{n=1}^\infty \alpha_n=\infty, \quad \sum_{n=1}^\infty \alpha_n^2 <\infty. 
		$$
Then the sequence $\set{z_n}$ obtained from SHB (\ref{eq:shb1}) and SNAG (\ref{eq:snag}) almost surely satisfies the assumptions of Lemma \ref{lem:chain}. Hence its limit set satisfies the conclusion of Lemma \ref{lem:chain} almost surely. 
\end{proposition}

\begin{proof}
For SHB, write
$$
z_{n+1}  = z_n + \frac{\alpha_n}{1-\beta} \left(-\nabla f(z_n) + (\nabla f(x_n)-g_n) + (\nabla f(z_n) -\nabla f(x_n)) \right).
$$
Let $u_n=\nabla f(x_n)-g_n$ and $b_n=\nabla f(z_n) -\nabla f(x_n)$. Then Lemma \ref{lem:conv} implies that $b_n\ra 0$ as $n\ra \infty$ and $\sum_{i=1}^n \alpha_i u_i $ converges. It follows that $$
    \lim_{n\ra \infty} \sup_{k\ge n+1} \norm{\sum_{i=n}^{k-1}\alpha_iu_i} = 0. 
    $$
Boundedness of $\set{z_n}$ follows from the fact that $\nabla f(z_n)\ra 0$ as $n\ra \infty$ (Lemma \ref{lem:conv}) and Assumption \ref{as:nonflat}.

Hence the assumptions of Lemma \ref{lem:chain} are met and its conclusion follows. The proof for SNAG (\ref{eq:snag}) is almost identical and therefore omitted. 
\end{proof}

\begin{remark} 
Since $z_n=x_n - \frac{\beta}{1-\beta}v_n$ and $v_n\ra 0$, the limit sets of $\set{x_n}$ and $\set{z_n}$ coincide and hence both enjoy the property stated in the conclusion of Lemma \ref{lem:chain}. 
\end{remark}

\subsection{Lyapunov analysis around strict saddle manifold}\label{sec:lyap}

The saddle avoidance analysis relies on the construction of a Lyapunov function around the saddle manifold due to \cite[Proposition 9.5]{benaim1999dynamics}. %

In this section, we assume that $f$ is three times continuously differentiable. Since $\mathcal{S}$ is a strict saddle manifold, the center manifold theorem \citep{robinson2012introduction,shub1987global} implies that there exists a submanifold $\M$ of $\Real^d$, namely the center stable manifold of $\S$, that is locally invariant under the flow $\set{\Phi_t}$ %
in the sense that there exists a neighborhood $\U$ of $\S$ and a positive time $t_0$ such that $\Phi_t(\U\cap \M) \subset \M$ for all $\abs{t}\le t_0$. Furthermore, for each $x^*\in \S$, we have $R^d = T_{x^*}\M\oplus E^u_{x^*}$, where $E^u_{x^*}$ is the unstable subspace of $\Real^n$ for (\ref{eq:gdflow}) at $x^*$. Due to the assumption on the strict saddle manifold, the dimension of $E^u_{x^*}$ is at least one and the dimension of $\M$ is at most $d-1$. Relying on center manifold theory and geometric arguments, one can construct a Lyapunov function $V$ based on the following function
$$
\rho(y) = \norm{\Pi(y)-y},
$$
which maps from a neighborhood $\U_0$ of $\S$ to $\Real_{\ge 0}$, where $\Pi(y)$ projects $y$ on $\M$ along the unstable directions of (\ref{eq:gdflow}). The following result was proved in \cite{benaim1999dynamics} (see also \cite{mertikopoulos2020almost} for discussions more specific to strict saddle manifold as defined by Definition \ref{def:saddle}).

\begin{proposition}[\cite{benaim1999dynamics}] \label{prop:lyap}
    There exists a compact neighborhood $\mathcal{U}_{\mathcal{S}}$ of $\mathcal{S}$ and positive constants $\tau$ and $c$ 
  such that 
the function $V:\,\mathcal{U}_\mathcal{S}\ra\Real$ given by
$$
V(x) = \int_0^\tau \rho\left(\Phi_{-t}(x)\right)dt,
$$
where $\set{\Phi_t}$ is the flow generated by (\ref{eq:gdflow}), satisfies the following properties:
\begin{enumerate}
	\item $V$ is twice continuously differentiable on $\U_\S\setminus\M$. For all $x\in \U_\S\cap \M$, $V$ admits a right derivative $DV(x):\,\Real^d\ra\Real^d$ which is Lipschitz, convex, and positively homogeneous. 
	\item For all $x\in \U_\S$, 
	$$
	DV(x)[-\nabla F(x)] \ge c V(x). 
	$$

	\item There exists a positive %
	constant $C$ such that, for all $x\in \U_\S$, 
        \begin{equation}\label{eq:DV}
        DV(x)[v] \ge -C \norm{v},     
        \end{equation}
	for all $v\in\Real^d$.
		
	\item There exists a constant $\gamma>0$ and a neighborhood $V$ of the origin of $\Real^d$ such that for all $x\in \U_\S$ and $v\in V$, we have
	\begin{equation}
		V(x+v) \ge V(x) + DV(x)[v] - \frac{\gamma}{2}\norm{v}^2.
	\end{equation}

	\item There exists a constant $m>0$ such that for all $x\in \U_\S\setminus \M$, 
	$$
	\norm{\nabla V(x)} \ge m,
	$$
	and for all $x\in \U_\S\cap \M$ and $v\in \Real^d$,
	$$
	DV(x)[v]\ge m\norm{v-D\Pi(x)v}. 
	$$
\end{enumerate}
\end{proposition}

For more details on this construction and proof of the above proposition\footnote{Proposition \ref{prop:lyap}(3) was not explicitly stated in \cite{benaim1999dynamics}, but can be easily derived from (35) and (36) in the proof of \cite[Proposition 9.5]{benaim1999dynamics}.}, readers are referred to \cite[Proposition 9.5]{benaim1999dynamics} (see also \cite[Appendix C]{mertikopoulos2020almost}). For more background information on the topic, we refer the readers to \cite{benaim1995dynamics,benaim1999dynamics,lee2012introduction,shub1987global,robinson2012introduction}.

\subsection{Almost sure saddle avoidance analysis}\label{sec:analysis}

In this section, we analyze almost sure avoidance of any strict saddle manifold. The following lemma due to \cite[Lemma 5.5]{pemantle1992vertex} plays an important role in the probabilistic argument of the proof. 

\begin{lemma}[\cite{pemantle1992vertex}]\label{lem:pem}
	Let $\set{S_n}$ be a nonnegative stochastic process defined as $S_n=S_0+\sum_{i=1}^n Z_i$, where $\set{Z_n}$ is adapted to a filtration $\set{\mathcal{F}_n}$. Suppose that $\set{\alpha_n}$ satisfies $\alpha_n=\Theta\left(\frac{1}{n^p}\right)$, where $\frac12<p\le 1$. Suppose there exist positive constants $b_1$, $b_2$, and $b_3$ such that the following hold almost surely for all $n$ sufficiently large: %
	\begin{enumerate}
		\item $\norm{Z_{n+1}}\le b_1 \alpha_n$;
		\item $\mathbf{1}_{\set{S_n>b_2\alpha_n}} \E[Z_{n+1}\mid \mathcal{F}_n]\ge 0$; 
		\item $\E[S_{n+1}^2-S_n^2\mid \mathcal{F}_n]\ge b_3\alpha_n^2$. 
	\end{enumerate}
Then $\mathbb{P}(\lim_{n\ra\infty}S_n=0)=0$.
\end{lemma}

The lemma was proved in \cite{pemantle1992vertex} for $p=1$, but the proof for $\frac12<p\le 1$ is the same. The same result was proved and used in \cite{pemantle1990nonconvergence} but not explicitly stated. See \cite[Lemma 9.6]{benaim1999dynamics} for a more general form of this result. 

We need another technical lemma that states when the sequence $\set{x_n}$ is uniformly bounded, then by Assumption \ref{as:localb}, we can obtain a uniform rate of convergence by $v_n$ to zero, at least for $\set{\alpha_t}$ chosen as in Lemma \ref{lem:pem}.

\begin{lemma}\label{lem:estv}
Let $\set{x_n}$ and $\set{v_n}$ be obtained from SHB (\ref{eq:shb2}) or SNAG (\ref{eq:snag2}) with $\set{\alpha_n}$ satisfying $\alpha_n=\Theta(\frac{1}{n^p})$, where $\frac12<p\le 1$. Suppose that there exists some $K>0$ such that $\norm{x_n}\le K$ for all $n\ge 1$. Then $v_n=\O(\frac{1}{n^{p}})$. 
\end{lemma}

\begin{proof}
For SHB, we have 
\begin{align*}
	\norm{v_{n+1}}^2 = \beta^2 \norm{v_n}^2 - 2\beta\alpha_n \inner{g_n,v_n} + \alpha_n^2 \norm{g_n}^2. 
\end{align*}
For SNAG, we have
\begin{align*}
	\norm{v_{n+1}}^2 = \beta^2 \norm{v_n}^2 - 2\beta^2\alpha_n \inner{g_n,v_n} + \beta^2\alpha_n^2 \norm{g_n}^2. 
\end{align*}
In either case, using the elementary inequality that $ \inner{a,b}\le \eps\norm{a}^2+\frac{1}{\eps}\norm{b}^2$, we can find two constants $\eps>0$ and $C>0$, where both only depend on $\beta$ and $\eps>0$ can be made arbitrarily small, such that 
\begin{align*}
	\norm{v_{n+1}}^2 \le (\beta^2+\eps) \norm{v_n}^2 +C\alpha_n^2 \norm{g_n}^2. 
\end{align*}
We can choose $\eps$ such that $\beta^2+\eps<1$. Let $\lambda=\beta^2+\eps$. By Assumption \ref{as:localb}, there exists another constant $C_1$ such that 
\begin{align*}
	\norm{v_{n+1}}^2 \le \lambda \norm{v_{n}}^2  +C_1\alpha_n^2. 
\end{align*}
First, observe that the above implies 
$$
\norm{v_{n+1}}^2 -\norm{v_{n}}^2 \le  C_1\alpha_n^2. 
$$
Summing both sizes from $1$ to $m$ shows  $\norm{v_{m+1}}^2\le \norm{v_1}^2+ C_1\sum_{i=1}^{m}C_1\alpha_n^2$, which shows that $\set{v_n^2}$ is uniformly bounded, provided the uniform bound on $\set{x_n}$. 

To show a specific rate estimate for $\set{v_n}$, we claim that for $n$ sufficiently large, $v_n^2=\O\left(\frac{1}{n^{2p}}\right)$, i.e., there exists some constant $C_2$ such that $v_n^2\le \frac{C_2}{n^{2p}}$.  Since  $\alpha_n^2=\Theta\left(\frac{1}{n^{2p}}\right)$, there exist positive constants $A$ and $B$ such that 
$$
\frac{A_1}{n^{2p}}\le \alpha_n^2 \le \frac{B_1}{n^{2p}},
$$
for all $n$. Fix any $\mu\in (\lambda,1)$. Choose $C_2$ such that $C_2\ge \frac{C_1B_1}{\mu-\lambda}$. By induction, suppose $v_n^2 \le \frac{C_2}{n^{2p}}$ holds for some $n$ such that $\frac{n^{2p}}{(n+1)^{2p}}\ge \mu$ (which also holds for all subsequent $n$). We have
\begin{align*}
	v_{n+1}^2 &\le \lambda \norm{v_{n}}^2  +C_1\alpha_n^2 \le  \frac{\lambda C_2}{n^{2p}} + \frac{C_1B_1}{n^{2p}} \\ 
	& = \frac{C_2}{(n+1)^{2p}} - \frac{C_2}{(n+1)^{2p}} + \frac{\lambda C_2}{n^{2p}} + \frac{C_1B_1}{n^{2p}} \\
	&\le \frac{C_2}{(n+1)^{2p}} - \frac{\mu C_2- \lambda C_2 - C_1B_1}{n^{2p}}\\
	&\le \frac{C_2}{(n+1)^{2p}},
\end{align*}
by the choice of $C_2$. Hence the estimate holds for all $n$ sufficiently large. 
\end{proof}

With these preliminary results, we are ready to prove Theorem \ref{thm:main}. 

\ \\
\noindent\textbf{Proof of Theorem \ref{thm:main}}

Let $\U_\S$ be the neighborhood defined in Proposition \ref{prop:lyap}. Consider the sequences $\set{x_n}$ and $\set{z_n}$ generated from SHB or SNAG. Without loss of generality, assume $z_0=x_0\in \U_\S$ (the proof for $z_N\in \U_\S$ for any $N$ is identical). For any $k\ge 1$, define the stopping time 
$$
T_\S^k = \set{n\ge 0:\,z_n\not\in \U_\S\text{ or }\norm{x_n}>k},
$$
which is the first exit time of $\set{z_n}$ from $\U_\S$ or $\set{x_n}$ from the $k$-radius ball. Define two sequences of random variables $\set{Z_n}$ and $\set{S_n}$ as follows\footnote{Note that we should have a superscript $k$ on $\set{Z_n}$ and $\set{S_n}$ as they depend on $k$, but we omit it to simplify the notation.}:
\begin{equation}
	Z_{n+1} = (V(z_{n+1}) - V(z_{n}))\mathbf{1}_{\set{n\le T_\S^k}} + \alpha_n\mathbf{1}_{\set{n> T_\S^k}},
\end{equation}
and
\begin{equation}
	S_0 = V(z_0),\quad S_n = S_0 + \sum_{i=1}^nZ_n.
\end{equation}
Clearly, if $T_\S=\infty$, then $S_n=V(z_n)$ for all $n\ge 0$ by telescoping. 

We verify that $\set{Z_n}$ and $\set{S_n}$ defined above satisfy the conditions of Lemma \ref{lem:pem}. 

\textbf{Condition 1:} It is clearly satisfied if $n>T_\S^k$. If $n\le T_\S^k$, since $V$ is locally Lipschitz and the stochastic gradient is locally bounded (Assumption \ref{as:localb}), we have $\norm{Z_{n+1}}\le b_1\alpha_n$ for some $b_1>0$. 

\textbf{Condition 2:} If $n>T_\S^k$, we have $Z_{n+1}=\alpha_n$ and 
\begin{equation}\label{eq:lowest0}
\mathbf{1}_{\set{n>T_\S^k}} \E[Z_{n+1}\mid \mathcal{F}_n]\ge \mathbf{1}_{\set{n>T_\S^k}}\alpha_n \ge 0. 
\end{equation}
If $n\le T_\S^k$, we have $z_n\in \U_\S$ and, by Proposition \ref{prop:lyap}, 
\begin{align}
Z_{n+1} &= V(z_{n+1}) - V(z_{n}) \ge DV(z_n)[-\frac{\alpha_n}{1-\beta}g_n] -\frac{\gamma \alpha_n^2}{2(1-\beta)^2}\norm{g_n}^2\notag\\
&  \ge \frac{\alpha_n}{1-\beta} DV(z_n)[-\nabla f(z_n)] + \frac{\alpha_n}{1-\beta} DV(z_n)[\nabla f(z_n)-g_n]  -\frac{\gamma \alpha_n^2}{2(1-\beta)^2}\norm{g_n}^2\notag\\
& \ge \frac{\alpha_n c}{1-\beta} V(z_n)  + \frac{\alpha_n}{1-\beta} DV(z_n)[\nabla f(z_n)-g_n]  -C_1\alpha_n^2, \label{eq:zn}
\end{align}
where $C_1>0$ is a constant that can be derived from the bound $k$ for $\set{x_n}$ and Assumption \ref{as:abc}. Taking the conditional expectation w.r.t. $\mathcal{F}_n$ gives
\begin{align}
	\E[Z_{n+1}\mid\mathcal{F}_n] & \ge \frac{\alpha_n c}{1-\beta} V(z_n)  + \frac{\alpha_n}{1-\beta} \E[DV(z_n)[\nabla f(z_n)-g_n]\mid\mathcal{F}_n]  -C_1\alpha_n^2. \label{eq:lowerest}
\end{align}
Now we can use convexity of the right derivative of $V$ (Proposition \ref{prop:lyap}) and the conditional Jensen's inequality to obtain 
\begin{align*}
\E[DV(z_n)[\nabla f(z_n)-g_n]\mid\mathcal{F}_n] & \ge DV(z_n)[\nabla f(z_n)-\E[g_n \mid\mathcal{F}_n ]] \\
&  \ge DV(z_n)[\nabla f(z_n)-\nabla f(x_n)]\\
&  \ge -C_2\norm{v_n},
\end{align*}
where $C_2>0$ is a constant that can be derived from Proposition \ref{prop:lyap}, the Lipschitz continuity of $\nabla f$, and (\ref{eq:zv}). Putting this back to (\ref{eq:lowerest}) and using Lemma \ref{lem:estv}, we obtain 
\begin{align}
	\E[Z_{n+1}\mid\mathcal{F}_n] & \ge \frac{\alpha_n c}{1-\beta} V(z_n)  -\frac{\alpha_n C_2\norm{v_n}}{1-\beta} -C_1\alpha_n^2 \ge \frac{\alpha_n c}{1-\beta} V(z_n)- C_3\alpha_n^2,  \label{eq:lowerest2}
\end{align}
for some $C_3>0$. In other words, we have shown
\begin{align}
	\mathbf{1}_{\set{n\le T_\S^k}}\E[Z_{n+1}\mid\mathcal{F}_n] & \ge  \mathbf{1}_{\set{n\le T_\S^k}}\left(\frac{\alpha_n c}{1-\beta} V(z_n)- C_3\alpha_n^2\right). \label{eq:lowerest3}
\end{align}
Clearly, if we choose $b_2=\frac{C_3(1-\beta)}{c}$, then $S_n=V(z_n)>b_2\alpha_n$ implies $\E[Z_{n+1}\mid\mathcal{F}_n]\ge 0$. Condition 2 is verified. 

\textbf{Condition 3:} We have
\begin{align*}
	\E[S_{n+1}^2-S_n^2\mid \mathcal{F}_n] & = \E [Z_{n+1}^2 + 2S_n Z_{n+1} \mid \mathcal{F}_n]\\
	& = \E [Z_{n+1}^2  \mid \mathcal{F}_n] + 2S_n\E [  Z_{n+1} \mid \mathcal{F}_n].
\end{align*}
If $S_n > b_2\alpha_n$, condition 2 implies that $\E [  Z_{n+1} \mid \mathcal{F}_n]\ge 0$ and hence the right-hand side of the above equation is non-negative. If $S_n \le  b_2\alpha_n$, it follows from (\ref{eq:lowest0}) and (\ref{eq:lowerest3}) that 
$$
2S_n\E [  Z_{n+1} \mid \mathcal{F}_n] \ge -2b_2C_3\alpha_n^3.
$$
Hence, to verify condition 3, it suffices to show that there exists a constant $b_4>0$ such that 
$$
\E [Z_{n+1}^2  \mid \mathcal{F}_n] \ge b_4\alpha_n^2,
$$
for all $n$ sufficiently large. If $n>T_\S^k$, this obviously holds. For $n\le T_\S^k$, we investigate $\E[Z_{n+1}^+ \mid \mathcal{F}_n]$. In view of Jensen's inequality
$$
\E [Z_{n+1}^2  \mid \mathcal{F}_n] \ge \E[Z_{n+1}^+ \mid \mathcal{F}_n]^2,
$$
we only need to show $\E[Z_{n+1}^+ \mid \mathcal{F}_n]=\Omega(\alpha_n)$. Consider two cases: (i) $z_n\in \M$; (ii) $z_n\not \in \M$. If $z_n\not\in \M$, the right derivative in (\ref{eq:zn}) becomes the gradient and from it we obtain 
\begin{align}
	Z_{n+1} & \ge \frac{\alpha_n}{1-\beta} \inner{\nabla V(z_n),(\nabla f(z_n)-\nabla f(x_n)) + (\nabla f(x_n) -g_n)}  -C_1\alpha_n^2\notag\\
	&\ge -\frac{\alpha_n C_2}{1-\beta}\norm{v_n} + \frac{\alpha_n}{1-\beta} \inner{\nabla V(z_n),\nabla f(x_n) -g_n} - C_1\alpha_n^2\notag \\
	&\ge \frac{\alpha_n}{1-\beta} \inner{\nabla V(z_n),\nabla f(x_n) -g_n} - C_3\alpha_n^2, \label{eq:zn2}
\end{align}
where $C_1$, $C_2$, and $C_3$ are as defined above in the proof for condition 2. Taking conditional expectation on the positive part, we obtain 
\begin{align}
	\E [Z_{n+1}^+ \mid \mathcal{F}_n] 
	&\ge \frac{\alpha_n}{1-\beta} \E [\inner{\nabla V(z_n),\nabla f(x_n) -g_n}^+] - C_3\alpha_n^2\notag\\
	&\ge \frac{\alpha_n}{1-\beta}\norm{\nabla V(z_n)}b- C_3\alpha_n^2\notag\\
	&\ge \frac{\alpha_n m b}{1-\beta}- C_3\alpha_n^2,  \label{eq:zn3}
\end{align}
where we used Assumption \ref{as:uniform} on the unit vector $\nabla V(z_n)/\norm{\nabla V(z_n)}$ and then Proposition \ref{prop:lyap}. Hence we do have $\E[Z_{n+1}^+ \mid \mathcal{F}_n]=\Omega(\alpha_n)$ in this case. 

If $z_n\in \M$, we can choose a unit vector $u_n$ such that $\inner{u_n,y}$ for all $y\in T_{z_n}\M$.  Since $D\Pi(z_n)$ takes values in $T_{z_n}\M$ \cite[p.~51]{benaim1999dynamics}, we have
$
\inner{u_n,D\Pi(z_n)v}
$
for any $v\in\Real^d$. In view of (\ref{eq:zn}) and by Proposition \ref{prop:lyap}, we estimate
\begin{align}
DV(z_n)[\nabla f(z_n)-g_n] 
	&\ge m\norm{(f(z_n)-g_n) - D\Pi(z_n)[\nabla f(z_n)-g_n]} \notag\\
	&\ge \inner{u_n,(f(z_n)-g_n) - D\Pi(z_n)[\nabla f(z_n)-g_n]}\notag\\
	&= \inner{u_n,f(z_n)-g_n}, \label{eq:zn4}
\end{align}
where the first inequality is by Proposition \ref{prop:lyap}, the second one is Cauchy-Schwartz, and the equality is by the choice of $u_n$ above. Continuing from (\ref{eq:zn4}), we obtain
\begin{align}
\inner{u_n,f(z_n)-g_n} & = \inner{u_n,f(z_n)-\nabla f(x_n)} +  \inner{u_n,\nabla f(x_n)-g_n} \notag\\
& \ge -\frac{L\beta}{1-\beta}\norm{v_n} + \inner{u_n,\nabla f(x_n)-g_n}, \label{eq:zn5}
\end{align}
where we used Lipschitz continuity of $\nabla f$. Putting (\ref{eq:zn4}) and (\ref{eq:zn5}) in (\ref{eq:zn}) and using Lemma \ref{lem:estv} and Assumption \ref{as:uniform}, we obtain 
\begin{align*}
\E[Z_{n+1}^+ \mid\mathcal{F}_n]	&\ge \frac{\alpha_t}{1-\beta} \E [\inner{u_n,\nabla f(x_n)-g_n}^+] -C_4\alpha_n^2\ge \frac{\alpha_t }{1-\beta} b - C_4\alpha_n^2,
\end{align*}
for sufficiently large $n$, where $b>0$ is from Assumption \ref{as:uniform}, and $C_4$ can be derived from Lemma \ref{lem:estv}. Hence we have $\E[Z_{n+1}^+ \mid \mathcal{F}_n]=\Omega(\alpha_n)$ in the second case as well. 

Since conditions 1--3 of Lemma \ref{lem:pem} are verified, we conclude by Lemma \ref{lem:pem} that $\mathbb{P}(S_n\ra 0\text{ as }n\ra\infty)=0$. We prove that $\mathbb{P} (T_\S^k=\infty)=0$ for any $k$. Suppose that there exists some $k$ such that $\mathbb{P} (T_\S^k=\infty)>0$. For almost every path in $\set{T_\S^k=\infty}$, by Proposition \ref{prop:limset}, the limit set of $\set{z_n}$, denoted by $L(\set{z_n})$ forms an invariant subset of $\U_\S$ under the flow $\set{\Phi_t}$. Pick any limit point $z\in L(\set{z_n})\subset \U_\S$, we have $\Phi_t(y)\in L(\set{z_n})\subset \U_\S$ for all $t\ge 0$. By Proposition \ref{prop:lyap}, $V(\Phi_t(z))\ge e^{ct}V(z)$ for all $t>0$. Hence we must have $V(z)=0$. In other words, the sequence $S_n=V(z_n)\ra 0$. Since $\mathbb{P}(S_n\ra 0\text{ as }n\ra\infty)=0$, we must have $\mathbb{P} (T_\S^k=\infty)=0$. It follows that $\mathbb{P} (T_\S^k<\infty)=1$ for all $k$. Let $\Omega_0$ denote the event on which the conclusion of Lemma \ref{lem:conv} holds. Then $\mathbb{P}(\Omega_0)=1$. Let $B_k$ denote the event $\set{\sup_{n} \norm{x_n} \le k }\cap \Omega_0$. By Assumption \ref{as:nonflat}, almost every $\set{x_n}$ will be ultimately bounded, because $\nabla f(x_n)\ra 0$ as $n\ra \infty$ and $\lim\inf_{\norm{x}\ra\infty} \norm{f(x)}>0$. It follows that $\cup_{k=1}^\infty B_k=\Omega_0$. On each $B_k$, $T_\S^k<\infty$ implies that $\set{z_n}$ eventually exits $\U_\S$ (in fact infinitely often by repeating the argument in this proof). As a result, $z_n\not\ra \S$ as $n\ra \infty$ on $B_k$ for each $k$, and hence entirely on $\Omega_0$. The proof is complete.
\hfill $\blacksquare$

\section{Conclusions}

In conclusion, our study provides evidence for the effectiveness of various stochastic gradient descent methods, including SGD, SHB, and SNAG, in avoiding strict saddle points. Our analysis expands upon previous work on SGD by removing the requirement for bounded gradients and noise in the objective function, and instead relying on a more practical local boundedness assumption on the noisy gradient. The results of our study demonstrate that even with non-bounded gradients and noise, these methods can still converge to local minimizers. This research contributes to the understanding of the behavior of gradient descent methods in non-convex optimization and has potential implications for their use in solving a wide range of machine learning and optimization problems.

\bibliography{colt23}

\begin{thebibliography}{26}
\providecommand{\natexlab}[1]{#1}
\providecommand{\url}[1]{\texttt{#1}}
\expandafter\ifx\csname urlstyle\endcsname\relax
  \providecommand{\doi}[1]{doi: #1}\else
  \providecommand{\doi}{doi: \begingroup \urlstyle{rm}\Url}\fi

\bibitem[Bena{\"\i}m(1996)]{benaim1996dynamical}
Michel Bena{\"\i}m.
\newblock A dynamical system approach to stochastic approximations.
\newblock \emph{SIAM Journal on Control and Optimization}, 34\penalty0
  (2):\penalty0 437--472, 1996.

\bibitem[Bena{\"\i}m(1999)]{benaim1999dynamics}
Michel Bena{\"\i}m.
\newblock Dynamics of stochastic approximation algorithms.
\newblock In \emph{Seminaire de probabilites XXXIII}, pages 1--68. Springer,
  1999.

\bibitem[Bena{\"\i}m and Hirsch(1995)]{benaim1995dynamics}
Michel Bena{\"\i}m and Morris~W Hirsch.
\newblock Dynamics of morse-smale urn processes.
\newblock \emph{Ergodic Theory and Dynamical Systems}, 15\penalty0
  (6):\penalty0 1005--1030, 1995.

\bibitem[Bena{\"\i}m and Hirsch(1996)]{benaim1996asymptotic}
Michel Bena{\"\i}m and Morris~W Hirsch.
\newblock Asymptotic pseudotrajectories and chain recurrent flows, with
  applications.
\newblock \emph{Journal of Dynamics and Differential Equations}, 8:\penalty0
  141--176, 1996.

\bibitem[Brandi{\`e}re(1998)]{brandiere1998some}
Odile Brandi{\`e}re.
\newblock Some pathological traps for stochastic approximation.
\newblock \emph{SIAM Journal on Control and Optimization}, 36\penalty0
  (4):\penalty0 1293--1314, 1998.

\bibitem[Brandi{\`e}re and Duflo(1996)]{brandiere1996algorithmes}
Odile Brandi{\`e}re and Marie Duflo.
\newblock Les algorithmes stochastiques contournent-ils les pi{\`e}ges?
\newblock \emph{Annales de l'IHP Probabilit{\'e}s et statistiques}, 32\penalty0
  (3):\penalty0 395--427, 1996.

\bibitem[Daneshmand et~al.(2018)Daneshmand, Kohler, Lucchi, and
  Hofmann]{daneshmand2018escaping}
Hadi Daneshmand, Jonas Kohler, Aurelien Lucchi, and Thomas Hofmann.
\newblock Escaping saddles with stochastic gradients.
\newblock In \emph{International Conference on Machine Learning}, pages
  1155--1164. PMLR, 2018.

\bibitem[Du et~al.(2017)Du, Jin, Lee, Jordan, Singh, and
  Poczos]{du2017gradient}
Simon~S Du, Chi Jin, Jason~D Lee, Michael~I Jordan, Aarti Singh, and Barnabas
  Poczos.
\newblock Gradient descent can take exponential time to escape saddle points.
\newblock \emph{Advances in Neural Information Processing systems}, 30, 2017.

\bibitem[Fang et~al.(2019)Fang, Lin, and Zhang]{fang2019sharp}
Cong Fang, Zhouchen Lin, and Tong Zhang.
\newblock Sharp analysis for nonconvex sgd escaping from saddle points.
\newblock In \emph{Conference on Learning Theory}, pages 1192--1234. PMLR,
  2019.

\bibitem[Ge et~al.(2015)Ge, Huang, Jin, and Yuan]{ge2015escaping}
Rong Ge, Furong Huang, Chi Jin, and Yang Yuan.
\newblock Escaping from saddle points—online stochastic gradient for tensor
  decomposition.
\newblock In \emph{Conference on Learning Theory}, pages 797--842. PMLR, 2015.

\bibitem[Jin et~al.(2017)Jin, Ge, Netrapalli, Kakade, and
  Jordan]{jin2017escape}
Chi Jin, Rong Ge, Praneeth Netrapalli, Sham~M Kakade, and Michael~I Jordan.
\newblock How to escape saddle points efficiently.
\newblock In \emph{International Conference on Machine Learning}, pages
  1724--1732. PMLR, 2017.

\bibitem[Khaled and Richt{\'a}rik(2020)]{khaled2020better}
Ahmed Khaled and Peter Richt{\'a}rik.
\newblock Better theory for sgd in the nonconvex world.
\newblock \emph{arXiv preprint arXiv:2002.03329}, 2020.

\bibitem[Lee et~al.(2016)Lee, Simchowitz, Jordan, and Recht]{lee2016gradient}
Jason~D Lee, Max Simchowitz, Michael~I Jordan, and Benjamin Recht.
\newblock Gradient descent only converges to minimizers.
\newblock In \emph{Conference on Learning Theory}, pages 1246--1257. PMLR,
  2016.

\bibitem[Lee et~al.(2019)Lee, Panageas, Piliouras, Simchowitz, Jordan, and
  Recht]{lee2019first}
Jason~D Lee, Ioannis Panageas, Georgios Piliouras, Max Simchowitz, Michael~I
  Jordan, and Benjamin Recht.
\newblock First-order methods almost always avoid strict saddle points.
\newblock \emph{Mathematical programming}, 176:\penalty0 311--337, 2019.

\bibitem[Lee(2012)]{lee2012introduction}
John Lee.
\newblock \emph{Introduction to Smooth Manifolds}.
\newblock Springer Science \& Business Media, 2012.

\bibitem[Liu and Yuan(2022)]{liu2022almost}
Jun Liu and Ye~Yuan.
\newblock On almost sure convergence rates of stochastic gradient methods.
\newblock In \emph{Conference on Learning Theory}, pages 2963--2983. PMLR,
  2022.

\bibitem[Mertikopoulos et~al.(2020)Mertikopoulos, Hallak, Kavis, and
  Cevher]{mertikopoulos2020almost}
Panayotis Mertikopoulos, Nadav Hallak, Ali Kavis, and Volkan Cevher.
\newblock On the almost sure convergence of stochastic gradient descent in
  non-convex problems.
\newblock \emph{Advances in Neural Information Processing Systems},
  33:\penalty0 1117--1128, 2020.

\bibitem[Nesterov(2003)]{nesterov2003introductory}
Yurii Nesterov.
\newblock \emph{Introductory Lectures on Convex Optimization: A Basic Course}.
\newblock Springer Science \& Business Media, 2003.

\bibitem[Panageas and Piliouras(2017)]{panageas2017gradient}
Ioannis Panageas and Georgios Piliouras.
\newblock Gradient descent only converges to minimizers: Non-isolated critical
  points and invariant regions.
\newblock In \emph{Innovations in Theoretical Computer Science Conference (ITCS
  2017)}, volume~67, page~2. Schloss Dagstuhl--Leibniz-Zentrum fuer Informatik,
  2017.

\bibitem[Panageas et~al.(2019)Panageas, Piliouras, and Wang]{panageas2019first}
Ioannis Panageas, Georgios Piliouras, and Xiao Wang.
\newblock First-order methods almost always avoid saddle points: The case of
  vanishing step-sizes.
\newblock \emph{Advances in Neural Information Processing Systems}, 32, 2019.

\bibitem[Pemantle(1990)]{pemantle1990nonconvergence}
Robin Pemantle.
\newblock Nonconvergence to unstable points in urn models and stochastic
  approximations.
\newblock \emph{The Annals of Probability}, 18\penalty0 (2):\penalty0 698--712,
  1990.

\bibitem[Pemantle(1992)]{pemantle1992vertex}
Robin Pemantle.
\newblock Vertex-reinforced random walk.
\newblock \emph{Probability Theory and Related Fields}, 92\penalty0
  (1):\penalty0 117--136, 1992.

\bibitem[Robinson(2012)]{robinson2012introduction}
Rex~Clark Robinson.
\newblock \emph{An Introduction to Dynamical Systems: Continuous and Discrete},
  volume~19.
\newblock American Mathematical Society, 2012.

\bibitem[Shub(1987)]{shub1987global}
Michael Shub.
\newblock \emph{Global Stability of Dynamical Systems}.
\newblock Springer, 1987.

\bibitem[Vlaski and Sayed(2022)]{vlaski2022second}
Stefan Vlaski and Ali~H. Sayed.
\newblock Second-order guarantees of stochastic gradient descent in nonconvex
  optimization.
\newblock \emph{IEEE Transactions on Automatic Control}, 67\penalty0
  (12):\penalty0 6489--6504, 2022.
\newblock \doi{10.1109/TAC.2021.3131963}.

\bibitem[Vlatakis-Gkaragkounis et~al.(2019)Vlatakis-Gkaragkounis, Flokas, and
  Piliouras]{vlatakis2019efficiently}
Emmanouil-Vasileios Vlatakis-Gkaragkounis, Lampros Flokas, and Georgios
  Piliouras.
\newblock Efficiently avoiding saddle points with zero order methods: No
  gradients required.
\newblock \emph{Advances in Neural Information Processing systems}, 32, 2019.

\end{thebibliography}

\end{document}